\documentclass[11pt]{article}

\usepackage{fullpage}

 \usepackage{makeidx} 

\usepackage{authblk}

\usepackage{algorithm,algorithmic}
\usepackage{enumitem}
\usepackage{amsmath,amssymb,amsthm}

\newtheorem{theorem}{Theorem}
\newtheorem{lemma}[theorem]{Lemma}
\newtheorem{corollary}[theorem]{Corollary}

\newtheorem{definition}[theorem]{Definition}

\newcommand{\cC}{\mathcal{C}}

\begin{document}

%\mainmatter

\title{Data Stability in Clustering: A Closer Look}

\author[1]{Shalev Ben-David}
\author[2]{Lev Reyzin}
\affil[1]{Department of Electrical Engineering and Computer Science \authorcr Massachusetts Institute of Technology \authorcr
77 Massachusetts Avenue \authorcr
Cambridge, MA 02139 \authorcr
\texttt{shalev@mit.edu} \authorcr $ $}
\affil[2]{Department of Mathematics, Statistics, and Computer Science \authorcr
University of Illinois at Chicago \authorcr
851 South Morgan Street \authorcr
Chicago, IL 60607 \authorcr
\texttt{lreyzin@math.uic.edu} }

\date{}
\maketitle

%%%%%%%%%%%%%%%%%%%%%%%%%%%%%%%%%%%%%%%%%%%%%%%
 
\begin{abstract}
We consider the model
introduced by Bilu and Linial~\cite{BL}, who study problems for which 
the optimal clustering does not change when
distances are perturbed. 
They show that even when a problem is NP-hard, 
it is sometimes possible to obtain efficient algorithms for instances resilient to 
certain multiplicative perturbations, e.g.\ on the order of $O(\sqrt{n})$ for max-cut clustering. 
Awasthi~et~al.~\cite{ABS10} consider
center-based objectives, and Balcan and Liang~\cite{Liang} analyze the $k$-median 
and min-sum objectives, giving efficient algorithms for instances resilient to 
certain {constant} multiplicative perturbations.

Here, we are motivated by the question of to what extent these assumptions can be relaxed
while allowing for efficient algorithms.
We show there is little room to improve these results
by giving NP-hardness lower bounds for 
both the $k$-median and min-sum objectives. 
On the other hand, we show that constant multiplicative resilience parameters
can be so strong as to make the clustering problem trivial,
leaving only a narrow range of resilience parameters for which clustering is interesting.
We also consider a model of additive perturbations and give a correspondence 
between additive and multiplicative notions of stability.
Our results provide a close examination of the consequences of assuming stability in data.

\end{abstract}

%%%%%%%%%%%%%%%%%%%%%%%%%%%%%%%%%%%%%%%%%%%%%%%

\section{Introduction}

Clustering is one of the most widely-used techniques in statistical data analysis.
The need to partition, or cluster, data into meaningful categories naturally arises in
virtually every domain where data is abundant.
Unfortunately, most of the natural clustering objectives, including $k$-median,
$k$-means, and min-sum, are NP-hard to optimize~\cite{GuhaK99,jain_new_2002}.
It is, therefore, unsurprising that many of the clustering algorithms used in practice come with
few guarantees.

Motivated by overcoming the hardness results, Bilu and Linial~\cite{BL} consider a
\textbf{perturbation resilience assumption} that they argue is often 
implicitly made when choosing
a clustering objective: that the optimum clustering
to the desired objective $\Phi$ is preserved under multiplicative perturbations up to a factor
$\alpha > 1$ to the distances between the points.
They reason that if the optimum clustering to an objective $\Phi$ is not resilient, as in, if small
perturbations to the distances can cause the optimum to change, then $\Phi$ may have been the
wrong objective to be optimizing in the first place. Bilu and Linial~\cite{BL} show that for max-cut clustering, instances resilient to perturbations of 
$\alpha = O(\sqrt{n})$ have efficient algorithms for recovering the optimum itself.

Continuing that line of research, Awasthi~et~al.~\cite{ABS10} give a polynomial time algorithm that 
finds the optimum clustering for instances resilient to multiplicative
perturbations of $\alpha = 3$ for center-based\footnote{For center-based clustering objectives, the clustering is defined by a choice of centers, and the objective is a function of the distances of the points to their closest center.} 
clustering objectives when centers must 
come from the data (we call this the \textbf{proper} setting), and
$\alpha = 2 + \sqrt{3}$ when when the centers do not need to (we call this the \textbf{Steiner} setting).
Their method relies on a \textbf{stability} property implied by perturbation resilience
(see Section~\ref{prelim}). For the Steiner case, they also prove an NP-hardness lower bound of $\alpha = 3$.
Subsequently, Balcan and Liang~\cite{Liang} consider the proper setting
and improve the constant past $\alpha = 3$ 
by giving a new polynomial time algorithm for the $k$-median objective for 
$\alpha = 1+\sqrt{2} \approx 2.4$ stable instances.

\subsection{Our results}

Our work further delves into the proper setting, for which no lower bounds have previously been shown for the 
stability property.
In Section~\ref{sec:lbs} we show that even in the proper case, where the algorithm is restricted to
choosing its centers from the data, 
for any $\epsilon > 0$,
it is NP-hard to optimally cluster $(2-\epsilon)$-stable instances, both for the \textbf{$k$-median} and
\textbf{min-sum} objectives (Theorems~\ref{thm:kclb} and~\ref{thm:mslb}).  
To prove this for the min-sum objective, we define a new notion of stability
that is implied by perturbation resilience, a notion that  may be of independent interest.

Then in Section~\ref{sec:conseq}, 
we look at the implications of assuming resilience or stability in the data, 
even for a constant perturbation parameter $\alpha$.  
We show that for even fairly small constants, 
the data begins to have very strong structural properties, as to make the clustering task fairly trivial.  
When $\alpha$ exceeds $2 + \sqrt{3}$, the data begins to
show what is called \textbf{strict separation}, where each point is closer to points in its own cluster
than to points in other clusters (Theorem~\ref{theorem:strictseparation}).

Finally, in Section~\ref{app:additivestability},
we look at whether the picture can be improved for clustering data that is stable under additive, 
rather than multiplicative, perturbations.  
One hope would be that \textbf{additive stability} is a more useful assumption,
where a polynomial time algorithm for $\epsilon$-stable instances might be possible.  
Unfortunately, this is not the
case.  We consider a natural additive model and show that severe lower bounds hold for 
the additive notion as well (Theorems~\ref{thm:addklb} and~\ref{thm:addmslb}).
On the positive side, we show via reductions that algorithms for multiplicatively stable data also work
for additively stable data for a different but related parameter.

Our results demonstrate that on the one hand, it is hard to improve the algorithms to
work for low stability constants, and that on the other hand, higher stability constants can be quite strong,
to the point of trivializing the problem.  Furthermore, switching from a multiplicative to an additive
stability assumption does not help to circumvent the hardness results, and perhaps makes matters worse.  
These results, taken together, narrow the range of interesting parameters for theoretical study
and highlight the strong role that the choice of constant plays in stability assumptions.

One thing to note that there is some difference between the very related resilience and
stability properties (see Section~\ref{prelim}), stability being weaker and more general~\cite{ABS10}. 
Some of our results apply to both notions, and some only to stability.  This still leaves
open the possibility of devising polynomial-time
algorithms that, for a much smaller $\alpha$, 
work on all the $\alpha$-perturbation resilient instances, but not on all $\alpha$-stable ones.

\subsection{Previous work}\label{sec:prevwork}

We examine previous work on stability, both as a data dependent assumption in clustering and in other
settings.

\subsubsection{Stability as a data assumption in clustering}

The classical approach in theoretical computer science to dealing with the worst-case NP-hardness of clustering
has been to develop efficient approximation algorithms for the various clustering
objectives~\cite{AroraRR98,AryaGKMMP04,BartalCR01,CharikarGTS02,KumarSS10,VegaKKR03},
and significant efforts have been exerted to improve approximation ratios and to prove lower
bounds.   In particular, for metric $k$-median, the best known guarantee is a
$(3+\epsilon)$-approximation~\cite{AryaGKMMP04}, and the best known lower bound is
$(1+1/e)$-hardness of approximation~\cite{GuhaK99,jain_new_2002}.
For metric min-sum, the best known result is a $O(\mathrm{polylog}(n))$-approximation 
to the optimum~\cite{BartalCR01}.

In contrast, a more recent direction of research has been to characterize under what conditions we can
find a desirable clustering efficiently.
Perturbation resilience/stability are such conditions,
but they are related to other stability notions in clustering.
Ostrovsky~et~al.~\cite{OstrovskyRSS06} demonstrate the effectiveness of Lloyd-type
algorithms~\cite{Lloyd82} on instances with the stability property that the cost of the optimal $k$-means
solution is small compared to the cost of the optimal $(k-1)$-means solution, and
their guarantees have later been improved by Awasthi~et~al.~\cite{AwasthiBS10}.

In a different line of work, Balcan~et~al.~\cite{BalcanBV08} consider what stability properties of a 
similarity function, with respect to the ground truth clustering, are sufficient to cluster well.
In a related direction,
Balcan~et~al.~\cite{BBG} argue that,
for a given objective $\Phi$,
approximation algorithms
are most useful when the clusterings they produce are structurally
close to the optimum originally sought in choosing to optimize $\Phi$ in the first place.
They then show that, for many objectives,
if one makes this assumption explicit -- that all $c$-approximations to the objective yield
a clustering that is  $\epsilon$-close to the optimum -- then
one can recover an $\epsilon$-close clustering in polynomial time, even
for values of $c$ below the hardness of approximation constant.
The assumptions and algorithms of Balcan~et~al.~\cite{BBG} have subsequently
been carefully analyzed by Schalekamp~et~al.~\cite{SchalekampYZ10}.

Ackerman and Ben-David~\cite{AckermanB09} also study various notions of resilience, and
among their results, introduce a notion of stability similar to the one studied herein, except 
only the positions of cluster centers are perturbed.
Their notion is strictly weaker -- i.e.\ any perturbation resilient instance is also stable
in their framework.
They show that Euclidean instances
stable to perturbations of cluster centers will have polynomial algorithms for finding near-optimal
clusterings.
However, they require the desired number of clusters to be small compared to the input size:
their algorithms have running times exponential in the number of clusters.
In this work, we do not make that assumption; this means our positive results are more general,
but our lower bounds don't apply.

\subsubsection{Stability in other settings}

Just as the Bilu and Linial~\cite{BL} notion of stability gives conditions under which efficient
clustering is possible, similar concepts have been studied in game theory. 
Lipton~et~al.~\cite{LiptonMM06} propose a notion of stability for solution concepts of games.
They define a game to be stable if small perturbations to the payoff matrix do not significantly
change the value of the game, and they show games are generally not stable under this definition.
Then, in a similar spirit to the work of Bilu and Linial, 
Awasthi~et~al.~\cite{AwasthiBBSV10} propose a related stability condition for a game,
which can be leveraged in
finding its approximate Nash equilibria.

The Bilu and Linial~\cite{BL} notion of stability has also been studied in the context of the metric traveling salesman
problem, for which Mihal{\'a}k~et~al.~\cite{MihalakSSW11} give efficient algorithms for $1.8$-perturbation 
resilient instances, illustrating another case where a stability assumption can circumvent NP-hardness.

From a different direction, Ben-David~et~al.~\cite{Ben-DavidLP06} consider
the stability of clustering algorithms, as opposed to instances.
They say an algorithm is stable if it produces similar clusterings for different inputs drawn from
the same distribution.  They argue that stability is not as useful a notion as had been previously 
thought in determining
various parameters, such as the optimal number of clusters.

\section{Notation and preliminaries}\label{prelim}
In a clustering instance, we are given a set $S$ of $n$ points in a finite metric space, and we denote
$d: S \times S \rightarrow \mathbb{R}_{\geq 0}$ as the distance function. $\Phi$ denotes the objective
function over a partition of $S$ into $k$ clusters which we want to optimize over the metric, i.e.\ $\Phi$ assigns
a score to every clustering.
The optimal clustering with respect to $\Phi$ is denoted as $\mathcal{C} = \{C_1,C_2,\dots, C_k\}$.

The \textbf{$k$-median objective} requires
$S$ to be partitioned into $k$ disjoint subsets $\{S_1,\dots, S_k\}$ and each subset $S_i$
to be assigned a center $s_i\in S$. The goal is to minimize $\Phi_\mathrm{med}$, 
measured by
$$\Phi_\mathrm{med}(S_1, \ldots, S_k) \doteq \sum_{i=1}^k\sum_{p\in S_i}d(p,s_i).$$
The centers in the optimal clustering are denoted as $c_1,\dots,c_k$.
In an optimal solution, each point is assigned to its nearest center.

For the \textbf{min-sum objective}, $S$ is partitioned into $k$ disjoint subsets,
and the objective is to minimize $\Phi_\mathrm{m-s}$, measured by
 $$\Phi_\mathrm{m-s}(S_1, \ldots, S_k) \doteq \sum_{i=1}^k\sum_{p, q\in S_i}d(p,q).$$

Now, we define the perturbation resilience notion introduced by Bilu and Linial~\cite{BL}.

\begin{definition}
For $\alpha > 1$, a clustering instance $(S,d)$ is \textbf{$\alpha$-perturbation resilient} to a given objective
$\Phi$ if for any function $d': S \times S \rightarrow \mathbb{R}_{\geq 0}$
such that $\forall p,q \in S,$ $$d(p,q) \leq d'(p,q) \leq \alpha d(p,q),$$ the %there is a unique 
optimal clustering
$\mathcal{C'}$ for $\Phi$ under $d'$ %and this clustering 
is equal to the optimal clustering
$\mathcal{C}$ for $\Phi$ under $d$.
\end{definition}

In this paper, we consider the $k$-median and min-sum objectives, and
we thereby investigate the following definitions of stability,
which are implied by perturbation resilience, as shown in
Sections~\ref{section-kmedian} and~\ref{minsumlb}.
The following definition is adapted from Awasthi~et~al.~\cite{ABS10}.

\begin{definition}\label{def:centerstable}
A clustering instance $(S,d)$ is \textbf{$\alpha$-center stable} for the $k$-median
objective if for any optimal cluster $C_i\in\mathcal{C}$ with center $c_i$, $C_j\in \mathcal{C}\ (j\neq i)$
with center $c_j$,  any point $p\in C_i$ satisfies $$\alpha d(p,c_i) < d(p,c_j).$$%\label{def:centerstability}
\end{definition}
Next, we define a new analogous notion of stability for the min-sum objective,
and we show in Section~\ref{minsumlb} that for the min-sum objective, perturbation resilience implies
min-sum stability.  
To help with exposition for the min-sum objective, 
we define the distance from a point $p$ to a set of points $A$,
$$d(p, A) \doteq \sum_{q\in A} d(p,q).$$
\begin{definition}\label{def:minsumstable}
A clustering instance $(S,d)$ is \textbf{$\alpha$-min-sum stable} for the
min-sum objective if for all optimal clusters
$C_i,C_j\in\mathcal{C}\ (j \neq i)$, any point $p \in C_i$ satisfies 
$$\alpha d(p,C_i) < d(p,C_j).$$\label{def:minsumstability}
\end{definition}
This is a useful generalization because, as we shall see,
known algorithms working under the perturbation resilience assumption 
can also be made to work under the weaker notion of min-sum stability.

\section{Lower bounds}\label{sec:lbs}

\subsection{The $k$-median objective}\label{section-kmedian}

Awasthi~et~al.~\cite{ABS10} prove the following connection between perturbation resilience and
stability. Both their algorithms and the algorithms of Balcan and Liang~\cite{Liang} crucially use this
stability assumption.

\begin{lemma}\label{kmedian-lemma-basic}
Any clustering instance that is $\alpha$-perturbation resilient for the $k$-median objective also
satisfies the $\alpha$-center stability.
\end{lemma}

Awasthi~et~al.~\cite{ABS10} proved that for $\alpha < 3 - \epsilon$, $k$-median 
clustering $\alpha$-center stable instances is NP-hard when Steiner points are allowed in the 
data.  Afterwards, Balcan and Liang~\cite{Liang} circumvented this lower bound and achieved
a polynomial time algorithm for $\alpha = 1 + \sqrt{2}$ by assuming the 
algorithm must choose cluster centers from within the data.

In the theorem below, we prove a lower bound for the center stable property in this more restricted setting, showing there is little hope of progress even for data 
where each point is nearly twice closer to its own center than to any other.

\begin{theorem}\label{thm:kclb}
For any $\epsilon > 0$, the problem of solving $(2 - \epsilon)$-center stable $k$-median
instances is NP-hard.
\end{theorem}

\begin{proof} 
We reduce from what we call the perfect dominating set promise problem (PDS-PP), 
which we prove to be NP-hard (see Appendix), 
where we are promised that the input graph $G = (V,E)$ is such that all of its
smallest dominating sets $D$ are perfect, and we are asked to find a dominating set of size
at most $d$.
The reduction is simple.
We take an instance of the NP-hard problem PDS-PP on $G=(V,E)$ on
$n$ vertices and reduce it to an $\alpha =  2 - \epsilon$-center stable instance.
Our distance
metric as follows.
Every vertex $v \in V$ becomes a point in the $k$-center instance.
For any two vertices $(u,v) \in E$ we define $d(u,v) = 1/2$. When $(u,v) \notin E$,
we set $d(u,v) = 1$.
This trivially satisfies the triangle inequality for any graph $G$, as the sum of the distances
along any two edges is at least $1$.  We set $k = d$.

We observe that a $k$-median solution of cost $(n - k)/2$ corresponds to a dominating
set of size $d$ in the PDS-PP instance, and is therefore NP-hard to find.
We also observe that because all solutions
of size $\le d$ in the PDS-PP instance are perfect, each (non-center)
point in the $k$-median solution
has distance $1/2$ to exactly one (its own) center, and a distance of $1$
to every other center.  Hence, this instance is $\alpha =  (2 - \epsilon)$-center stable,
completing the proof.
\end{proof}

\subsection{The min-sum objective}\label{minsumlb}

Analogously to Lemma~\ref{kmedian-lemma-basic}, we can show that $\alpha$-perturbation
resilience implies our new notion of $\alpha$-min-sum stability.

\begin{lemma}\label{lem:mulptos}
If a clustering instance is $\alpha$-perturbation resilient, then it is also $\alpha$-min-sum stable.
\end{lemma}

\begin{proof}
Assume to the contrary that the instance is $\alpha$-perturbation resilient but is not
$\alpha$-min-sum stable.  Then, there exist clusters $C_i, C_j$ in the optimal solution
$\mathcal{C}$ and a point $p \in C_i$ such that $\alpha d(p,C_i) \ge d(p,C_j)$.
We perturb $d$ as follows.  We define $d'$ such that for all points $q \in C_i$, $d'(p,q)
= \alpha d(p,q)$, and for the remaining distances, $d' = d$.  Clearly $d'$ is an
$\alpha$-perturbation of $d$.

We now note that  $\mathcal{C}$ is not optimal under $d'$. Namely, we can create a cheaper
solution $\mathcal{C'}$ that assigns point $p$ to cluster $C_j$, and leaves the remaining clusters
unchanged, which contradicts optimality of $\mathcal{C}$. This shows that $\mathcal{C}$ is not the
optimum under $d'$ which contradicts the instance being $\alpha$-perturbation resilient.
Therefore we can conclude that if a clustering instance is $\alpha$-perturbation resilient,
then must also be $\alpha$-min-sum
stable.
\end{proof}

Moreover, we show in the Appendix that the min-sum algorithm of Balcan and Liang~\cite{Liang}, 
which requires $\alpha$ to be bounded from below by
$3\left(\frac{\max_{C \in \mathcal{C}}|C|}{\min_{C \in \mathcal{C}}|C|-1}\right)$, 
works with this more general condition.
This further motivates following bound.

\begin{theorem}\label{thm:mslb}
For any $\epsilon > 0$, the problem of finding an optimal min-sum $k$ clustering in
$(2 - \epsilon)$-min-sum stable instances is NP-hard.
\end{theorem}

\begin{proof}
Consider the  \textbf{triangle partition problem}.  
Let graph $G = (V,E)$ and $|V| = n = 3k$, and let each vertex have maximum
degree of $d=4$. The problem
of whether the vertices of $G$ can be partitioned into sets $V_1, V_2, \ldots, V_k$
such that each $V_i$ contains a triangle in $G$ is
NP-complete~\cite{Garey_johnson_1979}, even with the degree restriction~\cite{rooij_2011}.

We reduce the triangle partition problem to an $(2 - \epsilon)$-min-sum stable 
clustering instance.  The metric is as follows.
Every vertex $v \in V$ becomes a point in the min-sum instance.
For any two vertices $(u,v) \in E$ we define $d(u,v) = 1/2$. When $(u,v) \notin E$,
we set $d(u,v) = 1$.
This  satisfies the triangle inequality for any graph, as the sum of the distances
along any two edges is at least $1$.

Now we show that we can cluster this instance into $k$ clusters
such that the cost of the min-sum objective is exactly $n$
if and only if the original instance is a YES instance of triangle partition.
This follows
from two facts.  
\begin{enumerate}
\item A YES instance of triangle partition maps to a clustering into $k = n/3$
clusters of size $3$ with pairwise distances $1/2$, for a total
cost of $n$
\item A cost of $n$ is the best achievable because a balanced clustering with all minimum pairwise intra-cluster distances is optimal.
In particular,
$\frac{1}{2} \sum_i n_i (n_i-1)$ s.t. $\sum_i n_i = n$ is a lower bound on the cost of any $k$-clustering with cluster-sizes ${n_i}$.  
By convexity this expression is minimized at $n$.
\end{enumerate}

In the clustering from our reduction, each point has a sum-of-distances
to its own cluster of $1$. 
Now we examine the sum-of-distances of any point to other clusters.  A point
has two distances of $1/2$ (edges) to its own cluster, and because $d=4$, it can have at most two more distances of $1/2$ (edges)
into any other cluster, leaving the third distance to the other cluster to be $1$,
yielding a total cost of
$\ge 2$ into any other cluster.  Hence, it is
$(2 - \epsilon)$-min-sum stable.  
%So, an algorithm for solving
%$(2 - \epsilon)$-min-sum stable instances can be used to solve the triangle partition problem.
\end{proof}

We note that it is tempting to restrict the degree bound to $3$ in order to further
improve the lower bound.  Unfortunately, the triangle partition problem on
graphs of maximum degree $3$ is polynomial-time solvable~\cite{rooij_2011}, and
we cannot improve the factor of $2-\epsilon$ by restricting to graphs of degree $3$
in this reduction.

\section{Strong consequences of stability}\label{sec:conseq}

In Section~\ref{sec:lbs}, we showed that $k$-median clustering even $(2-\epsilon)$-center stable
instances is $NP$-hard. 
In this section we show that even for resilience to constant multiplicative perturbations of 
 $\alpha > 2 + \sqrt{3} \approx 3.7$, the data obtains a property referred to as
 \textbf{strict separation}, where all points are closer to all other points in their own
cluster than to points in any other cluster; this property is known to be helpful in 
clustering~\cite{BalcanBV08}. 

\subsection{Strict separation}

Now we prove the following theorem, which shows that even for relatively small multiplicative 
constants for $\alpha$, center stable, and therefore perturbation resilient, instances 
exhibit strict separation.

\begin{theorem}\label{theorem:strictseparation}
Let $\mathcal{C} = \{C_1, \ldots, C_k\}$ be the optimal clustering of an 
$\alpha$-center stable instance with $\alpha > 2+\sqrt{3}$.
Let $p, p' \in C_i$ and $q \in C_j$ ($i \neq j$), then $d(p,q)~>~d(p,p').$
\end{theorem}

\begin{proof}
Let $\cC$ be an $\alpha$-center stable clustering.
Let $p,p'$ have center $c_1$ in $\cC$
and let $q$ have center $c_2$, with $c_1\neq c_2$.
We have
\begin{equation}\label{firstequation}
d(p,p')\leq d(p',c_1) + d(p,c_1) \leq d(p',c_1) + \frac{1}{\alpha} d(p,c_2)
\end{equation}
where the second inequality follows from the definition of $\alpha$-center stability.
Note that
$$d(p',c_1)\leq \frac{1}{\alpha}d(p',c_2),$$
and subtracting $\frac{1}{\alpha}d(p',c_1)$ from both sides gives
$$\left(1-\frac{1}{\alpha}\right)d(p',c_1)\leq\frac{1}{\alpha}(d(p',c_2)-d(p',c_1))\leq\frac{1}{\alpha}d(c_1,c_2).$$
This then implies
\begin{eqnarray*}
d(p',c_1) &\leq& \frac{1}{\alpha-1}d(c_1,c_2) \\
&\leq &\frac{1}{\alpha-1}d(p,c_1) + \frac{1}{\alpha-1}d(p,c_2) \\
&\leq& \frac{1}{\alpha(\alpha-1)}d(p,c_2)+\frac{1}{\alpha-1}d(p,c_2)\\
&=& \frac{\alpha+1}{\alpha(\alpha-1)}d(p,c_2).
\end{eqnarray*}
Plugging this into Equation~\ref{firstequation} gives
\begin{eqnarray*}
d(p,p') &\leq& \frac{\alpha+1}{\alpha(\alpha-1)}d(p,c_2) + \frac{1}{\alpha}d(p,c_2)\\
&=& \frac{2\alpha}{\alpha(\alpha-1)}d(p,c_2)\\
&=& \frac{2}{\alpha-1}d(p,c_2).
\end{eqnarray*}
We also have
\begin{eqnarray*}
d(p,c_2) &\leq& d(p,q)+d(q,c_2) \\
&\leq & d(p,q) + \frac{1}{\alpha}d(q,c_1)\\
&\leq& d(p,q)+\frac{1}{\alpha}(d(p,q)+d(p,c_1))\\
&=& \frac{\alpha+1}{\alpha}d(p,q)+\frac{1}{\alpha}d(p,c_1)\\
&\leq& \frac{\alpha+1}{\alpha}d(p,q)+\frac{1}{\alpha^2}d(p,c_2),
\end{eqnarray*}
and subtracting $\frac{1}{\alpha^2}d(p,c_2)$ from both sides gives
$$\frac{\alpha^2-1}{\alpha^2}d(p,c_2)\leq \frac{\alpha+1}{\alpha}d(p,q)$$
or
$$d(p,c_2)\leq \frac{\alpha}{\alpha-1}d(p,q).$$
We then conclude
\begin{eqnarray*}
d(p,p') &\leq& \frac{2}{\alpha-1}d(p,c_2) \\
&\leq& \left( \frac{2}{\alpha-1} \right) \left( \frac{\alpha}{\alpha-1} \right) d(p,q)\\
&=& \frac{2\alpha}{(\alpha-1)^2}d(p,q).
\end{eqnarray*}
Thus we get $d(p,p')<d(p,q)$ if we set $\alpha$ such that $2\alpha<(\alpha-1)^2$. Solving this gives $\alpha> 2+\sqrt{3}$.
\end{proof}

The following example shows that Theorem~\ref{theorem:strictseparation} is tight, as 
in lower values of $\alpha$ cannot alone imply strict separation.
Consider the metric space $\mathbb{R}$, the line.
The example is given by $p'=0$, $c_1=\sqrt{3}$, $p=1+\sqrt{3}$,
$q=2+2\sqrt{3}$, and $c_2=3+3\sqrt{3}$, with $p$ and $p'$ belonging
to the cluster of $c_1$ and $q$ belonging to the cluster of $c_2$.
This example is $\alpha$-center stable for $\alpha = 2+\sqrt{3}$,
but it does not satisfy strict separation. In particular, this example
makes every inequality in the proof above tight when $\alpha = 2+\sqrt{3}$.

The strict separation property, however, is quite strong, as can be seen from the following Corollary.

\begin{corollary}\label{c:streaming}
Let $\mathcal{C} = \{C_1, \ldots, C_k\}$ be the optimal clustering of a 
$2+\sqrt{3}$-center stable instance.
Any algorithm that chooses
centers $\{c'_1, \ldots, c'_k\}$ such that $c'_i \in C_i$ induces the partition 
$\mathcal{C}$ when points are assigned to their closest centers.
\end{corollary}

In fact, Balcan~et~al.~\cite{BalcanBV08} show that in instances satisfying the strict separation
property, a simple ``Single Linkage" algorithm
will produce a tree, a pruning of which gives the optimal clustering.
Such a pruning can be found using dynamic programming to produce the optimal
clustering.  Recovering the optimum for $2 + \sqrt{3}$-resilient instances is not a new result, 
but is rather meant to illustrate the power of stability assumptions.
Balcan and Liang~\cite{Liang} give a more involved polynomial algorithm for finding
optima in $1+\sqrt{2}$-resilient instances.

\section{Additive stability}\label{app:additivestability}

So far, in this paper our notions of stability were
defined with respect to multiplicative perturbations.  Similarly, we can imagine an instance
being resilient with respect to additive perturbations.  Consider the following definition.

\begin{definition}\label{def:addstab}
Let $d: S \times S \rightarrow [0,1]$, and let $0 < \beta \le 1$.
A clustering instance (S, d) is \textbf{additive $\beta$-perturbation}
resilient to a given objective $\Phi$ if for any
function $d' : S \times S \rightarrow R \ge 0$ such that $\forall p, q \in S$, $$d(p, q) \le d'(p, q) \le
d(p, q) + \beta,$$ the %there is a unique 
optimal clustering $\mathcal{C'}$ for $\Phi$ under $d'$
%and this clustering 
is equal to the optimal clustering $\mathcal{C}$ for $\Phi$ under $d$.
\end{definition}

We note that in the definition above, we require all pairwise distances between points to
be at most $1$.  Otherwise, resilience to additive perturbations would be a
very weak notion, as the distances in most instances could be scaled
as to be resilient to arbitrary additive perturbations.

Especially in light of positive results for other additive stability notions~\cite{AckermanB09,Ben-David06},
one possible hope is that our hardness
results might only apply to the multiplicative case, and that we might be able to get polynomial time
clustering algorithms for instances resilient to arbitrarily small additive perturbations.
We show that this is unfortunately not the case -- we introduce notions of additive
stability, similar to Definitions~\ref{def:centerstable} and~\ref{def:minsumstable},
and for the $k$-median and min-sum
objectives, we show correspondences
between multiplicative and additive stability.

\subsection{The $k$-median objective}

Analogously to Definition~\ref{def:centerstable}, we can define a notion of additive $\beta$-center stability.

\begin{definition}
Let $d: S \times S \rightarrow [0,1]$, and let $0 \le \beta \le 1$.
A clustering instance $(S,d)$ is \textbf{additive $\beta$-center
stable} to the $k$-median objective if for any optimal cluster
$C_i \in \mathcal{C}$ with center $c_i$,  $C_j \in \mathcal{C}\ (j \neq i)$ with center $c_j$,
any point
$p \in C_i$ satisfies $$d(p,c_i) + \beta < d(p,c_j).$$
\end{definition}

We can now prove that perturbation resilience implies center stability.

\begin{lemma}\label{lem:addptos}
Any clustering instance satisfying additive $\beta$-perturbation resilience for the
$k$-median objective also satisfies additive $\beta$-center stability.
\end{lemma}

\begin{proof}
The proof is similar to that of Lemmas~\ref{kmedian-lemma-basic}.
We prove that for every point $p$ and its center $c_i$ in the optimal clustering
of an additive $\beta$-perturbation resilient instance, it holds that $d(p,c_j)
> d(p,c_i) + \beta$ for any $j \neq i$.

Consider an additive $\beta$-perturbation resilient clustering instance.
Assume we blow up all the pairwise distances within cluster $C_i$ by an additive factor of
$\beta$.
As this is a legitimate perturbation of the distance function,
the optimal clustering under this perturbation is the same as the original one.
Hence, $p$ is still assigned to the same cluster.
Furthermore, since the distances within $C_i$ were all changed by the same constant factor,
$c_i$ will remain the center of the cluster.
The same holds for any other optimal clusters.
Since the optimal clustering under the perturbed distances is unique it
follows that even in the perturbed distance function, $p$ prefers $c_i$ to $c_j$,
which implies the lemma.
\end{proof}

Now we prove a lower bound that 
shows that the task of clustering additive $(1/2 - \epsilon)$-center stable instances with respect to the $k$-median objective remains NP-hard.

\begin{theorem}\label{thm:addklb}
For any $\epsilon > 0$, the problem of finding an optimal $k$-median clustering in
additive $(1/2 - \epsilon)$-center stable instances is NP-hard.
\end{theorem}

\begin{proof}
We use the reduction in Theorem~\ref{thm:kclb}, in which the metric
satisfies the needed property that $d: S \times S \rightarrow [0,1]$.  We observe
that the instances from the reduction are additive $(1/2-\epsilon)$-center stable.
Hence, an algorithm for solving $k$-median on a $(1/2-\epsilon)$-center
stable instance can decide whether a PDS-PP instance contains a dominating
set of a given size, completing the proof.
\end{proof}

We now consider center stability, as in the multiplicative case.
We prove that additive center stability implies multiplicative center stability, and this gives us
the property that any algorithm for $\left(\frac{1}{1-\beta}\right)$-center stable instances
will work for additive $\beta$-center stable instances.

\begin{lemma}\label{lem:atom}
Any additive $\beta$-center stable clustering instance for
the $k$-median objective is also (multiplicative) $\left(\frac{1}{1-\beta}\right)$-center stable.
\end{lemma}

\begin{proof}
Let the optimal clustering be $C_1, \ldots, C_k$, with centers $c_1, \ldots, c_k$,
 of an additive $\beta$-center stabile clustering instance.  Let $p \in C_i$ and let $i \neq j$. 
From the stability property, 
\begin{equation}\label{eq:abovezero}
d(p,c_j)  >  d(p, c_i) + \beta \ge \beta.
\end{equation}
We also have\ $d(p,c_i)  <  d(p, c_j) - \beta$, from which we can see % \nonumber \\
$$\frac{1}{d(p, c_j) - \beta}  <  \frac{1}{d(p,c_i)}.$$
This gives us
%\begin{eqnarray}
%\frac{d(p, c_j) }{d(p, c_j) - \beta} & < & \frac{d(p, c_j)}{d(p,c_i)}\label{eq:mp1} \\
%\frac{1}{1 - \beta} & < & \frac{d(p, c_j)}{d(p,c_i)}. \label{eq:mp2}
%\end{eqnarray}
\begin{equation}\label{eq:mp2} 
\frac{d(p, c_j)}{d(p,c_i)} > \frac{d(p, c_j) }{d(p, c_j) - \beta}  \ge \frac{1}{1 - \beta}.
\end{equation}

Equation~\ref{eq:mp2} is derived as follows.  The middle term, for
$d(p, c_j) \ge \beta$ (which we have from Equation~\ref{eq:abovezero}), is
monotonically decreasing in $d(p, c_j)$.  Using $d(p, c_j) \le 1$ bounds it from below.
Relating the LHS to the RHS of Equation~\ref{eq:mp2} gives us the needed stability property.
\end{proof}

\subsection{The min-sum objective}

Here we define additive min-sum stability and prove the analogous
theorems for the min-sum objective.

\begin{definition}
Let $d: S \times S \rightarrow [0,1]$, and let $0 \le \beta \le 1$.
A clustering instance is \textbf{additive $\beta$-min-sum stable}
for the min-sum objective if for every point p in any optimal cluster $C_i$,
it holds that $$d(p,C_i) + \beta(|C_i|-1) < d(p,C_j).$$
\end{definition}

\begin{lemma}\label{lem:addptosminsum}
If a clustering instance is additive $\beta$-perturbation resilient, then it is also
additive $\beta$-min-sum stable.
\end{lemma}

\begin{proof}
Assume to the contrary that the instance is $\beta$-perturbation resilient but is not
$\beta$-min-sum stable.  Then, there exist clusters $C_i, C_j$ in the optimal solution
 $\mathcal{C}$ and a point $p \in C_i$ such that $d(p,C_i) + \beta(|C_i| -1) \ge d(p,C_j)$.
Then, we perturb $d$ as follows.  We define $d'$ such that for all points $q \in C_i$, $d'(p,q)
= d(p,q) + \beta$, and for the remaining distances $d' = d$.  Clearly $d'$ is a valid additive
$\beta$-perturbation of $d$.

We now note that $C$ is not optimal under $d'$.
Namely, we can create a cheaper
solution $\mathcal{C'}$ that assigns point $p$ to cluster $C_j$, and leaves the remaining clusters
unchanged, which contradicts optimality of $\mathcal{C}$.
This shows that $\mathcal{C}$ is not the optimum under $d'$ which is contradictory to the fact
that the instance is additive $\beta$-perturbation resilient.
Therefore we conclude
that if a clustering instance is additive $\beta$-perturbation resilient, then it is also additive
$\beta$-min-sum
stable.
\end{proof}

As with the $k$-median objective, we show that additive min-sum stability exhibits similar
lower bounds as in the multiplicative case. 

\begin{theorem}\label{thm:addmslb}
For any $\epsilon > 0$, the problem of finding an optimal min-sum clustering in
additive $(1/2 - \epsilon)$-min-sum stable instances is NP-hard.
\end{theorem}

\begin{proof}
We use the reduction in Theorem~\ref{thm:mslb}, in which the metric
satisfies the  property that $d: S \times S \rightarrow [0,1]$.   
The instances from the reduction are additive $(1/2-\epsilon)$-min-sum stable.
Hence, an algorithm for clustering a $(1/2-\epsilon)$-min-sum
stable instance can solve the triangle partition problem.
\end{proof}

Finally, as we did for the $k$-median objective, we can also reduce additive stability to multiplicative stability
for the min-sum objective.

\begin{lemma}\label{lem:msatom}
Let $t=\frac{\max_{C\in \mathcal{C}}|C|}{\min_{C\in \mathcal{C}}|C| -1}.$
Any additive $\beta$-min-sum stabile clustering instance for
the min-sum objective is also (multiplicative) $\left(\frac{1}{1-\beta/t}\right)$-min-sum stable.
\end{lemma}

\begin{proof}
Let the optimal clustering be $C_1, \ldots, C_k$ and let $p \in C_i$.
Let $i \neq j$. From the stability property, we have
\begin{eqnarray}
d(p,C_j) &>&  d(p, C_i) + \beta(|C_i|-1) \nonumber \\
&\ge& \beta(|C_i|-1).\label{eq:msabovezero}
\end{eqnarray}
We also have
$$ d(p,C_i)  <  d(p, C_j) - \beta(|C_i|-1).$$
Taking reciprocals and multiplying by $d(p,C_j)$, we get
\begin{eqnarray}
\frac{d(p, C_j)}{d(p,C_i)} & > & \frac{d(p, C_j) }{d(p, C_j) - \beta(|C_i|-1) }\nonumber\\
& \ge &  \frac{|C_j|}{|C_j| - \beta(|C_i|-1)}  \label{eq:msmp2}\\
& \ge & \frac{\max_{C\in \mathcal{C}}|C|}{\max_{C\in \mathcal{C}}|C_j| - \beta(\min_{C\in \mathcal{C}}|C|-1) }\nonumber\\ 
& \ge & \frac{1}{1 - \beta/t}. \label{eq:msstablfinal}
\end{eqnarray}

Equation~\ref{eq:msmp2} is derived as follows:
$d(p, C_j) \ge \beta(|C_i|-1)$ (which we have from Equation~\ref{eq:msabovezero}), is
monotonically decreasing in $d(p, C_j)$.  Observing $d(p, c_j) \le |C_j|$
bounds it from below.
Equation~\ref{eq:msstablfinal} gives us the needed property.
\end{proof}

\section{Discussion}

Our lower bounds, together with the structural properties implied by 
fairly small constants, illustrate the importance parameter settings play in  
stability assumptions.
These results make us wonder the degree to which the assumptions 
studied herein hold in practice; empirical study of real datasets is warranted.

Another interesting direction is to relax the assumptions. 
Awasthi~et~al.~\cite{ABS10} suggest considering stability under random, and not worst-case, perturbations.
Balcan and Liang~\cite{Liang} also study a relaxed version of the assumption, where 
perturbations can change the optimal clustering, but not by much.
It is open to what extent, and on what data, any of these approaches will yield practical improvements.

\section*{Acknowledgements} 
We thank Maria-Florina Balcan and Yingyu Liang for helpful discussions, and we appreciate
Shai Ben-David's, Avrim Blum's and Santosh Vempala's feedback on the writing.
We also especially thank Shai Ben-David for helping us find a bug in the ALT 2012 version
of this paper -- its claimed one pass streaming result was incorrect.

This work was supported in part by a Simons Postdoctoral Fellowship
in Theoretical Computer Science and by ARC while Lev Reyzin was at the Georgia Institute of
Technology.

\bibliography{paper}

\appendix
%\section*{APPENDIX}

\section{Dominating set promise problem}\label{app:pdspp}

A \textbf{dominating set} in a unweighted
graph $G = (V,E)$
is a subset $D \subseteq V$ of
vertices such that each vertex in $V \setminus D$ has a neighbor
in $D$.
A dominating set is \textbf{perfect} if each vertex in $D \setminus V$
has exactly one neighbor in $D$.
The problems of finding the smallest dominating set and smallest
perfect dominating set are NP-hard.

We introduce a related problem, called the \textbf{perfect dominating set promise problem}.
In this problem we are promised that the input graph is such that all its dominating sets of size
less at most $d$
are perfect, and we are asked to find a set of cardinality at most $d$.

First, we prove the following.

\begin{theorem}\label{thm:sdom}
The  \textbf{perfect dominating set promise problem} (PDS-PP) is NP-hard.
\end{theorem}

\begin{proof}
The \textbf{ $3$d matching problem} ($3$DM) is as follows: let $X, Y, Z$
be finite disjoint sets with
$m = |X| = |Y| = |Z|$.
Let $T$ contain triples $(x,y,z)$ with $x \in X, y \in Y, z \in Z$ with $L = |T|$.  $M \subseteq T$
is a perfect $3d$-matching if for any two triples $(x_1,y_1,z_1), (x_2,y_2,z_2) \in M$,
we have $x_1 \neq x_2, y_1 \neq y_2, z_1 \neq z_2$.  We notice that $M$ is a disjoint partition.
Determining whether a perfect $3d$-matching exists (YES vs.\ NO instance) in a
$3d$-matching instance is known to be NP-complete.

Now we reduce an instance of the $3$DM problem to PDS-PP on $G = (V,E)$.
For $3$DM elements $X$, $Y$, and $Z$ we construct vertices
$V_X$, $V_Y$, and $V_Z$, respectively.
For each triple in $T$ we construct a vertex in set $V_T$.
Additionally, we make an extra vertex $v$.
This gives $V = V_X \cup V_Y \cup V_Z \cup V_T \cup \{v\}$.
We make the edge set $E$ as follows.
Every vertex in $V_T$ (which corresponds to a triple) has an edge to the vertices
that it contains in the corresponding $3$DM instance (one in each of $V_X$, $V_Y$, and $V_Z$).
Every vertex in $V_T$ also has an edge to $v$.

Now we will examine the structure of the smallest dominating set $D$ in the constructed
PDS-PP instance. The vertex $v$ must belong to $D$ so that all vertices in $V_T$ are covered.
Then, what remains is to
optimally cover the vertices in $V_X \cup V_Y \cup V_Z$ -- the cheapest solution is to use $m$
vertices from $V_T$ , and this is precisely the 3DM problem, which is NP-hard.
Hence, any solution of size $d = m + 1$ for the PDS-PP instance gives a solution to the $3DM$ instance.

We also observe that such a solution makes a perfect dominating set.
Each vertex in $V_T \setminus D$ has one neighbor in $D$, namely $v$.  Each vertex in
$V_X \cup V_Y \cup V_Z$ has a unique neighbor in $D$, namely the vertex in $V_T$
corresponding to its respective set in the $3$DM instance.
\end{proof}

\section{Average linkage for min-sum stability}\label{app:avglink}

Here, we further support the claim that algorithms designed for $\alpha$-perturbation 
resilient instances
with respect to the min-sum objective can often be made to work for data satisfying the more general
$\alpha$-min-sum stability property.

\begin{algorithm}[!tbhp]
\caption{min-sum, $\alpha$ perturbation resilience}\label{algminsum}
\label{AverageLinkage}
\begin{algorithmic}
\STATE{\textbf{Input:} Data set $S$, distance function $d(\cdot, \cdot)$ on $S$, $\min_i |C_i|$.}
\STATE{\textbf{Phase 1:} Connect each point with its $\frac{1}{2}\min_i |C_i|$ nearest neighbors.}\label{componentConstruct}
\STATE{$\bullet$ Initialize the clustering $\mathcal{C'}$ with each connected component being a cluster.}
\STATE{$\bullet$ Repeat till only one cluster remains in $\mathcal{C'}$: merge clusters $C,C'$ in $\mathcal{C'}$ which minimize $d_{avg}(C,C')$.}\label{componentMerge}
\STATE{$\bullet$ Let $T$ be the tree with components as leaves and internal nodes corresponding to the merges performed.}
\STATE{\textbf{Phase 2:} Apply dynamic programming on $T$ to get the minimum min-sum cost pruning ${\cal \tilde{C}}$.}
\STATE{\textbf{Output:} Output ${\cal \tilde{C}}$.}
\end{algorithmic}
\end{algorithm}

One such algorithm is Algorithm~\ref{algminsum}.  Balcan and Liang~\cite{Liang} proved that  Algorithm~\ref{algminsum} 
correctly clusters instances for which the condition in Lemma~\ref{lem:tech} holds.  We can prove the condition indeed holds for 
$\alpha$-min-sum stable instances (their proof of the lemma 
holds for the more restricted class of perturbation-resilient instances).
To state the lemma, we first define the distance between two point sets, $A$ and $B$:
$$d(A,B) \doteq \sum_{p\in A}\sum_{ q\in B}d(p,q).$$

\begin{lemma}\label{lem:tech}
Assume the optimal clustering is $\alpha$-min-sum stable.
For any two different clusters $C$ and $C'$ in $\mathcal{C}$ and every
$A \subset C$,  $\alpha d(A, \bar{A}) < d(A, C').$
\end{lemma}

\begin{proof}
From the definition of $\alpha d(A, \bar{A})$, we have
\begin{eqnarray*}
\alpha d(A, \bar{A})  & = &  \alpha \sum_{p \in A}\sum_{q \in \bar{A}}d(p,q)\\
& \le &  \alpha \sum_{p \in A}\sum_{q \in C}d(p,q)\\
& = &  \sum_{p \in A}\alpha\sum_{q \in C}d(p,q)\\
& < & \sum_{p \in A}\sum_{q \in C'}d(p,q)\\ \label{e2} 
& = & d(A,C').
\end{eqnarray*}
The first inequality comes from $\bar{A} \subset C$ and the second by
definition of min-sum stability.
\end{proof}

This,  in addition to Lemma~\ref{lem:mulptos}, can be used to show their algorithm can be employed for min-sum
stable instances.

\end{document}